%% file: main.tex
\documentclass[runningheads]{llncs}
\usepackage{times}

\usepackage{hyperref}
\usepackage{bibnames}
\usepackage{graphicx}
\usepackage[pdftex,dvipsnames,table,xcdraw]{xcolor}  
\usepackage{amsmath,amssymb}
\usepackage{listings}
\usepackage{multirow}
\usepackage{longtable}
\usepackage{arydshln}
\usepackage{enumitem}
\usepackage{amsfonts}
\usepackage{multicol}
\usepackage{rotating}

\newcommand{\Mora}{{\tt Mora}}

\newcommand{\R}{{\mathbb{R}}}
\newcommand{\N}{{\mathbb{N}}}
\newcommand{\Normal}{{\mathcal{N}}}
\newcommand{\Unif}{{\mathcal{U}}}
\newcommand{\Gaussian}{{\mathcal{G}}}

\newcommand{\E}{{\mathbb{E}}}

\usepackage{algorithm}
\usepackage[noend]{algpseudocode}

\makeatletter
\def\BState{\State\hskip-\ALG@thistlm}
\makeatother

\usepackage{xargs}                      
\usepackage[colorinlistoftodos,prependcaption,textsize=tiny]{todonotes}

\newcommandx{\unsure}[2][1=]{\todo[linecolor=red,backgroundcolor=red!25,bordercolor=red,#1]{#2}}
\newcommandx{\change}[2][1=]{\todo[linecolor=blue,backgroundcolor=blue!25,bordercolor=blue,#1]{#2}}
\newcommandx{\info}[2][1=]{\todo[linecolor=OliveGreen,backgroundcolor=OliveGreen!25,bordercolor=OliveGreen,#1]{#2}}
\newcommandx{\improvement}[2][1=]{\todo[linecolor=Plum,backgroundcolor=Plum!25,bordercolor=Plum,#1]{#2}}
\newcommandx{\thiswillnotshow}[2][1=]{\todo[disable,#1]{#2}}

\newcommand{\ProbModel}{Prob-solvable}

\newcommand{\nin}{\not\in}
\usepackage{booktabs}
\usepackage{pifont}
\usepackage{capt-of}

\begin{document}

\title{Analysis of Bayesian Networks via Prob-Solvable Loops}
\titlerunning{ }
\authorrunning{ }
\author{ }
%
\author{Ezio Bartocci \and
Laura Kov{\'{a}}cs \and
Miroslav Stankovi\v{c}}
%
%
\institute{TU Wien, Austria
}

%
\maketitle              
%

\input abstract

\let\thefootnote\relax\footnotetext{
This research was supported by the Vienna Science and Technology Fund (WWTF)
under grant ICT19-018 (ProbInG), the 
ERC Starting Grant 2014 SYMCAR 639270 
and the Austrian FWF project W1255-N23.}

\input intro

\input{preliminaries}

\input{extending_probsolvable}

\input{encoding}

\input{solving_bn_problems}


\input{implementation}

\input{related}

\input{conclusion}

\bibliographystyle{splncs04}
\bibliography{references}

\end{document}

%% file: abstract.tex

\begin{abstract}

Prob-solvable loops are probabilistic programs with polynomial 
assignments over random variables and parametrised distributions, for which the full automation of 
moment-based invariant generation is decidable. In this paper we
extend Prob-solvable loops with new features essential for encoding
Bayesian networks (BNs). 
We show that various BNs,  such as 
discrete, Gaussian, conditional linear Gaussian and dynamic BNs, can be
naturally  encoded as Prob-solvable loops. Thanks to these
encodings, we can
automatically solve several BN related problems, including exact inference, sensitivity analysis, filtering and computing the expected  
number of rejecting samples in sampling-based procedures. We evaluate
our work on a number of BN benchmarks, using automated invariant
generation within Prob-solvable loop analysis.
\end{abstract}

%% file: intro.tex

\section{Introduction}\label{sec:intro}

Bayesian networks (BNs) are well-established probabilistic models  
widely adopted to represent complex systems and to reason about their intrinsic 
uncertain knowledge. BNs are graphically depicted as directed 
acyclic graphs (DAGs) whose nodes represent random variables and edges
 capture conditional dependencies. 
Since the seminal work of~\cite{Pearl1985329},
BNs have been extensively employed in several application domains
including machine learning~\cite{Heckerman08}, speech recognition~\cite{ZweigR98}, 
sports betting~\cite{Constantinou2012}, gene regulatory networks~\cite{Friedman2000}, 
diagnosis of diseases~\cite{Jiang2010} and finance~\cite{Neapolitan2010}.
Part of their success is due to the inherited Bayesian inference framework 
enabling the prediction about the likelihood that one of several known causes 
is responsible for the evidence of an observed event.

\begin{figure}[t]
  \includegraphics[width=\linewidth]{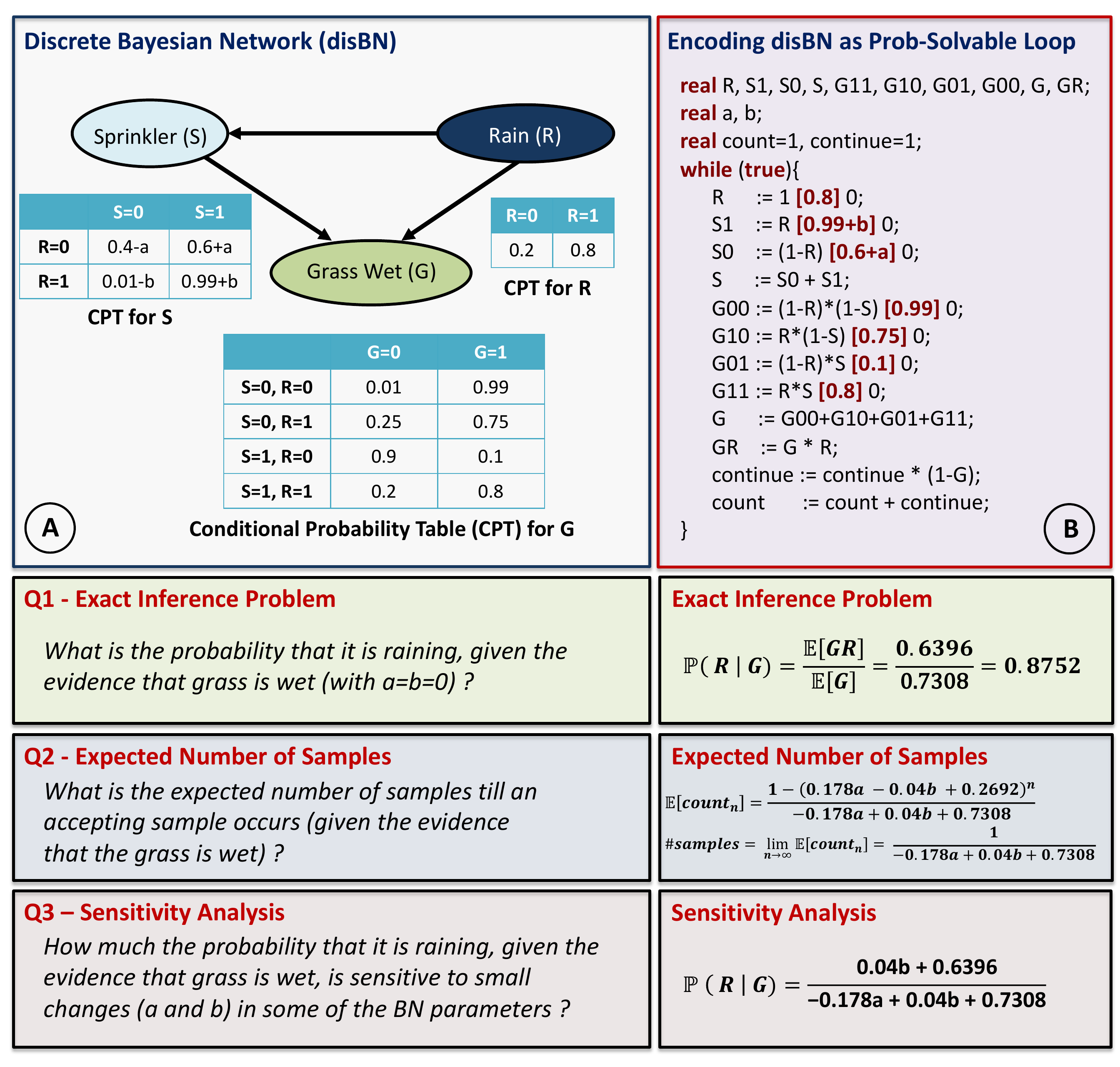}
  \vspace{-5ex}
  \caption{Solving probabilistic inference, the expected number of samples and the 
  sensitivity analysis for a discrete BN (disBN), by encoding the
    disBN as a \ProbModel{} loop 
  and computing automatically moment-based invariants (MBIs).}
  \vspace{-4ex}
  \label{fig:grass}
\end{figure}

Fig.~\ref{fig:grass} illustrates a simple BN with two events that can cause
the grass (G) to be wet: the rain (R)
or an active sprinkler (S). When it rains the sprinkler is usually not active, so the rain has a direct effect on the use 
of the sprinkler.  This dependency is provided by a \emph{conditional probability  
table}, in short CPT, associated to the sprinkler random variable
$S$.
A CPT lists, for each possible combination of values of the parents' 
variables (one for each row of the table), the corresponding probability  for the child's variable to have 
a certain discrete value (one for each column of the table). The
random variables $G, R, S$ of Fig.~\ref{fig:grass} are, for example, binary
random variables with Bernoulli conditional distributions. 
However, in general  BNs allow arbitrary types for their random
variables and their conditional distributions. 
\vspace{-1ex}
\noindent \paragraph{Probabilistic inference.} Given the BN in
Fig.~\ref{fig:grass}, the following can be asked: 
\vspace{-1ex}
$$\mbox{\emph{Q1 - What is the probability that it is raining, given that the grass is wet?}}$$
The answer to this question can be obtained by solving a
\emph{probabilistic inference}, that is the  
problem to optimally estimating the probability of an event given an 
observed evidence. 
The works in~\cite{Cooper90,DagumL93} show that both \emph{exact} and \emph{approximated} 
(up to an arbitrary precision) methods to solve probabilistic inference are NP-hard.
\vspace{-1ex}
\noindent \paragraph{How many samples?} Approximating solutions for 
probabilistic inferences can be done using Monte Carlo sampling techniques~\cite{Koller2009,Yuan2006}. For example, 
\emph{rejection sampling} is one of the fundamental techniques for sampling from the joint 
(unconditional) distribution of the BN: a sample is accepted when it complies
with the evidence, otherwise is rejected. Unfortunately, this method may require many samples before 
obtaining the first accepted samples, while most of the samples may be wasted simply because they 
do not satisfy the observations. Thus, an interesting question, investigated also in~\cite{howlong}, is:
\vspace{-1ex}
$$\mbox{\emph{Q2 - What is the expected number of samples until an accepting sample occurs?}}$$
\vspace{-5ex}
\noindent \paragraph{Sensitivity analysis.} As BN parameters are often
provided manually or estimated from (incomplete) data, they are most likely to be 
imprecise or wrong. For example, in Fig.~\ref{fig:grass} the CPT 
of the random variable $S$ contains imprecise symbolic parameters $a$ and $b$.
In this case, \emph{sensitivity analysis} aims to answer the following question: 
\vspace{-1ex}
$$\mbox{\emph{Q3 - How much does a small change in BN parameters 
affects probabilistic inference?}}$$
\vspace{-3ex}
\begin{figure}[t]
  \includegraphics[width=\linewidth]{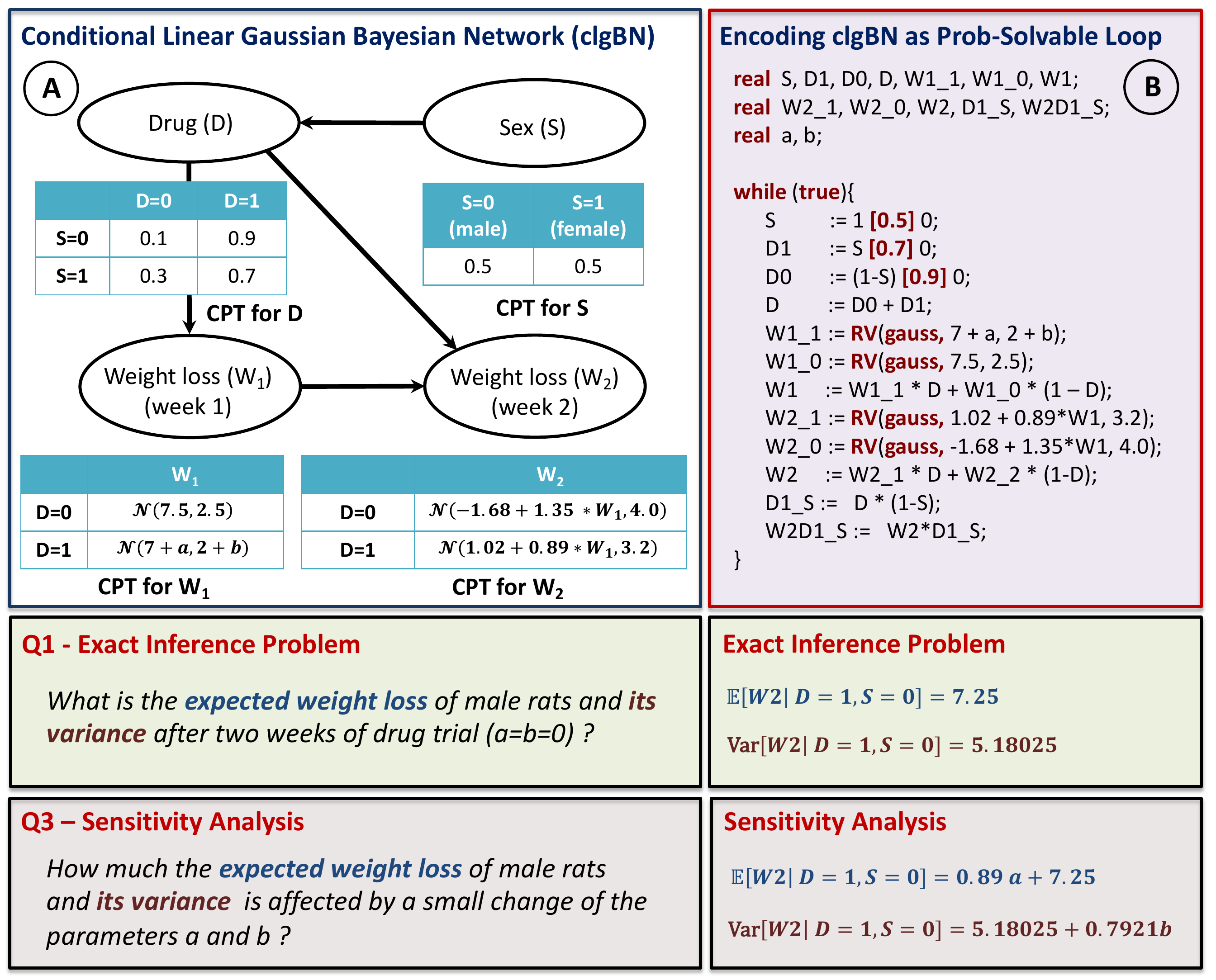}
  \caption{Solving probabilistic inference and sensitivity
    analysis in a conditional linear Gaussian BN (clgBN), by encoding the
    clgBN  as a \ProbModel{} loop 
  and  computing  MBIs.}
  \vspace{-3ex}
  \label{fig:rats}
\end{figure}
\vspace{-3ex}
\noindent \paragraph{Probabilistic Programs.} Probabilistic programs
(PPs) provide  
a unifying framework to both encode probabilistic graphical models, such as BNs, and to
implement sophisticated inference algorithms and decision making routines that can operate 
in real-world applications~\cite{Ghahramani15}. Probabilistic
programming languages, such as~\cite{TranHSBMB17,BinghamCJOPKSSH19,AiADGJKKTT19}  include native constructs 
for sampling distribution, enabling the programmer to mix
deterministic and stochastic elements. 
However, the automated analysis of PPs implemented in these languages
is still at its infancy.
For example, one of the main challenges in the analysis of PPs comes
with computing invariant properties summarizing PP loops.
While full automation of invariant generation for PPs is in general
undecidable, recent results identify classes of PPs for which
invariants can automatically be computed~\cite{howlong,probsolvable}.
In~\cite{probsolvable}, 
we introduced a method to automatically generate moment-based
invariants of so-called \ProbModel{} loops with polynomial assignments over random variables 
 and parametrised distributions. 
 Doing so, we exploit statistical properties to eliminate probabilistic choices and turn random updates into recurrence 
 relations over higher-order moments of program variables. 
 
   \vspace{-2ex}
 \noindent \paragraph{Analysis of  BNs as \ProbModel{} Loops.} In this
 paper we extend \ProbModel{} loops with new features essential for
 encoding BNs and for solving several kind of BN analysis
 via invariant generation over higher-order statistical moments  of
 \ProbModel{} loop variables.
 Fig.~\ref{fig:grass}(B) shows a \ProbModel{} loop encoding the probabilistic behaviour  
 of the discrete BN (disBN) illustrated in Fig.~\ref{fig:grass}(A). 
 The \ProbModel{} loop of Fig.~\ref{fig:grass}(B) requires  one
 variable for each disBN node, 
 one variable for each row of the CPT tables, one variable for each unknown parameter and 
  some extra variables that depend on the particular BN analysis. 
  For example, to solve exact probabilistic inference  and sensitivity
  analysis,  we require 
  an extra variable to store the product of the random variables $G$
  and $R$. On the other hand, to compute 
  the expected number of samples until an accepting sample occurs, we would need other two 
  auxiliary variables $count$ and $continue$. Each row of each CPT is encoded as a probabilistic 
  assignment in the \ProbModel{} loop. Our approach generates
  moment-based invariants as quantitative
  invariants over higher-order moments 
  to solve the three questions (Q1-Q3) of Fig.~\ref{fig:grass}. The
  required \ProbModel{} loop analysis requires however 
  additional  
  steps (e.g., calculating a limits) that are not yet supported
  in~\cite{probsolvable}. 
  Moreover, while the \ProbModel{} programming model
  of~\cite{probsolvable} can model the probabilistic 
  behavior of disBNs, it cannot model other BN variants, 
  such as BNs with Gaussian conditional dependencies as in
  Fig.~\ref{fig:rats}(A).  We therefore extend \ProbModel{}
  loops with new features supporting 
  Gaussian and uniform   
  random variables depending on other random
  variables (Section~\ref{sec:probsolvable}) and show that these extensions allow us to solve BN
  problems via \ProbModel{} loop reasoning (Sections~\ref{sec:encoding}-\ref{sec:BN problems}).

  
    \vspace{-2ex}
    \noindent \paragraph{Our contributions.} (i) We prove that our
    extended model of \ProbModel{} loops admits a decision procedure
  for computing moment-based invariants (Section~\ref{sec:probsolvable}. 
(ii) 
    We provide a sound encoding of BNs as \ProbModel{} loops, in
    particular addressing discrete BNs (disBNs), 
  Gaussian BNs (gBNs),  conditional linear Gaussian BNs (clgBNs) and
  dynamic BN (dynBNs) (Section~\ref{sec:encoding}.
  (iii) We formalize several BN problems as moment-based invariant
  generation tasks in \ProbModel{} loops (Section~\ref{sec:BN problems}).
  (iv) We implemented our
  approach in the 
  \textsc{Mora} tool~\cite{mora} and evaluated it on a number of examples, fully
  automating BN analysis via \ProbModel{} loop reasoning (Section~\ref{sec:implement}). 
%

%% file: preliminaries.tex

\section{Preliminaries}



We first introduce basic notions from statistics in order to reason about
probabilistic systems (Section~\ref{sec:prob}), and refer to~\cite{Lin92} for further
details. We then adopt basic definitions and properties of Bayesian
Networks (BNs) from~\cite{Pearl1985329} to our setting (Section~\ref{sec:BNs}). Throughout this paper, let $\N, \R$ denote
the set of natural and real numbers, respectively.

\subsection{Probability Space and Statistical Moments\label{sec:prob}}
We denote 
random variables by capital letters $X, Y, S, R, \ldots$ and program
variables by small letters $x, y, \ldots$, all possibly with indices.

\begin{definition}[Probability Space]
	A \emph{probability space} is a triple $(\Omega, F, P)$ consisting of
	a sample space $\Omega$ denoting the set of outcomes with
        $\Omega \not=\emptyset$, 
	 $F \subset 2^\Omega$ is a $\sigma$-algebra  denoting a set of
         events,  and  $P: F\rightarrow [0,1]$ is a probability measure
         with $P(\Omega)=1$.
\end{definition}

We now define random variables, together with their higher-order
statistical moments,
in order to reason about probabilistic properties. 

\begin{definition}[Random Variable]
	A \emph{random variable} $X: \Omega \rightarrow \R$ is
        a measurable function from a set $\Omega$ of 
	possible outcomes (also called sample space) to $\R$. 
        If $\Omega$ is finite or countable, the random variable $X$ is
        called \emph{discrete}; otherwise, $X$ is 
        \emph{continuous}. 
\end{definition}

In particular, in this paper we will be interested in the following random
variables: 
\begin{itemize}
	\item Random variable $X$ with Bernoulli distribution, given
          by probability $p$, where $\Omega = \{0, 1\}$ and $X(0) = 1-p$
          and $X(1) = p$; 
	\item Random variable $X$ with Gaussian distribution $\Gaussian(\mu,\sigma^2)$, given 
	by mean $\mu$ and variance $\sigma^2$, where $\Omega = \R$ 
	and $X(z) = \frac {1}{\sigma {\sqrt {2\pi }}}e^{-{\frac {1}{2}}\left({\frac {x-\mu }{\sigma }}\right)^{2}}$.
         \item Random variable $X$ with uniform distribution $\Unif(a,b)$, given 
	by lower and upper limits $a, b$ such that $a<b$, where $\Omega = \R$ 
	and $X(z) = 
\begin{cases}
\frac{1}{b-a} 		& \textrm{for } z\in[a,b]		\\
0							& \textrm{for } z\not\in[a,b]
\end{cases}	$.
\end{itemize}

\begin{example}
  The variables $R, S, G$ of the BN from
Fig.~\ref{fig:grass}(A) 
 are Bernoulli random variables, with variable $R$ given by
 probability $0.8$.  
Fig.~\ref{fig:rats}(A)  features two Bernoulli random variables $S$ and $D$
as well as two real-valued random variables $W1$ and $W2$ drawn from
a Gaussian 
distribution.
Note that the parameters of  the Gaussian distribution of $W1$ depend
on the values of $D$, whereas for $W2$ it depends on $D$ and $W1$.  
\end{example}

For a given random variable $X$ we will denote by $\Omega(X)$ the
sample space of $X$. 
When working with a random variable $X$, the most common statistical
moment of $X$ to consider is its first-order moment, called the 
expected value of $X$.

\begin{definition}[Expected Value]
An \emph{expected value of a random variable $X$} defined on a probability space $(\Omega, F, P)$ is 
the Lebesgue integral: $\E[X] = \int_\Omega X\cdot dP.$
In the special case \emph{when  $\Omega$ is discrete}, that is the
outcomes are ${X_1,\dots X_n}$ with corresponding 
probabilities ${p_1,\dots p_N}$ and $n\in\N$, we have $\E[X] = \sum_{i = 1}^n X_i\cdot p_i.$ 
The expected value of $X$ is often also referred to as the
\emph{mean} or $\mu$ of $X$.
\end{definition}

The key ingredient in analyzing and deriving properties of a random
variable $X$ is the so-called characteristic function of $X$.

\begin{definition}[Characteristic Function]
The \emph{characteristic function} of a random variable $X$, denoted
by $\phi_X(t)$, is the Fourier transform of its probability density
function (pdf). That is, $\phi(t) = \E[e^{itX}]$, with a bijective
relation between probability distributions and characteristic functions.
\end{definition}

The characteristic function $\phi_X(t)$ of a random variable $X$
captures 
the
value 
distribution induced by $X$. In particular, the 
characteristic function $\phi_X(t)$ of $X$ enables 
inferring properties about 
distributions given by weighted sums of $X$ and other random variables, and
thus also about 
statistical higher-order moments of $X$. 

%
\begin{definition}[Higher-Order Moments]
Let $X$ 
be a random variable, $c\in\R$ and $k\in\N$. We write
$Mom_{k}[X]$ to denote the \emph{$k$th raw moment of $X$},
which is defined as:  
\begin{equation}\label{eq:kthMoments}
  Mom_{k}[X]=\E[X^k].
 \end{equation}
\end{definition}

\begin{remark} For a Bernoulli random variable $X$ with
  parameter probability $p$,
  all  moments of $X$ coincide with its probability. Thus, 
  $Mom_{k}[X]=P(X=1)=p$.
\end{remark}

\begin{example}
  Fig.~\ref{fig:grass} lists the first-order moment $\E[G]$ of $G$, as well
  as the first-order moment $\E[GR]$ of the mixed random
  variable $GR$.
  The second-order moment of $W2$ is used to compute the variance 
  $Var(W2)$ of $W2$ in Fig.~\ref{fig:rats}. 
\end{example}



\input{pgms}

%% file: pgms.tex

\subsection{Probabilistic Graphical Models as Bayesian Networks}
\label{sec:BNs}


\begin{definition}[Bayesian Network (BN)] 
A \emph{Bayesian network (BN)} is a directed acyclic graph (DAG) in which each node 
corresponds to a discrete/random random variable. A set of directed 
links or arrows connects pairs of BN nodes. If there is an arrow from
a BN node $Y$ to a node 
$X$, then $Y$ is said to be a \emph{parent} of $X$.
\end{definition}

For a random variable/node $X$ in a BN, we write $Par(X)$ to denote the set of 
parents of $X$ in the BN. Each BN node $X$ has a \emph{conditional probability distribution $P(X | Par(X))$} 
that quantifies the effect of the parents $Par(X)$ on the node $X$. 
Dependencies in a BN can be given in different forms and we overview
the most common ones. For a discrete variable $X$, dependencies are often given 
by a conditional probability table, by listing all possible values of parent 
variables from $Par(X)$ and the corresponding values of $X$. 
In the case of a continuous variable $X$, dependencies can be  specified using Gaussian distributions. 
Another common dependency in a BN is a deterministic one, when value of a node $X$
is determined by values of its parents from $Par(X)$; that is, a binary variable can be true iff all 
its (binary) parents are true, or if one of its parents is true. 
We overview below  BN variants, studied further in Section~\ref{sec:encoding}.


\begin{definition}[Variants of Bayesian Networks]\label{def:BNvariants}
\begin{itemize}
\item  A \emph{discrete Bayesian Network (disBN)} is a BN whose variables are discrete-valued.
\item
  A \emph{Gaussian Bayesian Network (gBN)} is a BN whose dependencies
are given by the Gaussian distribution in which,  for any BN node $X$, we have $P(X | Par(X)) =\Gaussian(\mu_X, \sigma_X^2)$, with
$ \mu_X = \alpha_X + \sum_{k=1}^{m_X} \beta_{X, k} Y_{X, k},$, $Par(X)=\{Y_1, \cdots, Y_{m_X}\}$ and $\sigma_X^2$ is fixed.
\item 
A \emph{conditional linear Gaussian Bayesian Network (clgBN)} is a BN
in which (i) 
continuous nodes $X$ cannot be parents of discrete nodes $Y$; (ii) the
local distribution of each discrete node $Y$ is a conditional
probability table (CPT); (iii) the local distribution of each
continuous node $X$ is a set of Gaussian distributions, one for each
configurations of the discrete parents $Y$, with the continuous
parents acting as regressors.
\item A \emph{dynamic Bayesian Network (dynBN)} is a structured BN
  consisting of a series of time slices that represent the state of
  all the BN nodes $X$ at a certain time $t$. 
	For each time-slice, a dependency structure between the
        variables $X$ at that time is defined by intra-time-slice edges. 
	Additionally, there are edges between variables from different slices---inter-time-slice edges, with their directions following the direction of time.
\end{itemize}
\end{definition}


\begin{example}\label{ex:variantsBN}
A disBN encoding the probabilistic model of the grass getting wet is
shown in Fig.~\ref{fig:grass}(A). 
Fig.~\ref{fig:rats}(A) lists a clgBN, describing a weight loss process in a drug trial performed on rats.
The (Gaussian) random variables encoding weight loss for weeks $1$ and $2$
are respectively denoted with $W1$ and $W2$.
  %
%
\end{example}

%% file: extending_probsolvable.tex
\section{Programming Model: Extending \ProbModel{} Loops\label{sec:probsolvable}} 


We introduce our programming model extending the class of {\it
  \ProbModel{} loops}~\cite{probsolvable}, allowing us to encode
and analyze 
BN properties in Section~\ref{sec:encoding}. In
particular, we extend~\cite{probsolvable} to support \ProbModel{}
loops with symbolic random
variables encoding dependencies among other (random) variables,
where Gaussian and uniform random variables
can linearly depend on other program variables, encoding this way
common BN dependencies. 
To this end, we consider probabilistic while-programs as introduced
in~\cite{Kozen81,McIverM05} and restrict this class of programs to
probabilistic programs with polynomial updates
among random variables. We write $x:= e_1[p]e_2$ to denote that the
probability of the 
program variable $x$ being updated with expression $e_1$ is $p\in[0,1]$,
whereas the probability of $x$ being updated with expression $e_2$ is
$1-p$. 
In the sequel, whenever we
refer to a \ProbModel{} loop/program, we mean a program as
defined below.

\begin{definition}[\ProbModel{} Loop]\label{def:ProbModel}
Let $m\in\N$ and $x_1,\ldots x_m$ denote real-valued program
variables. A \emph{\ProbModel{} loop} with variables $x_1,\ldots
x_m$ is a probabilistic program of the form
\begin{equation}\label{eq:ProbModel}
  I; \texttt{while(true)} \{U\},  \qquad\qquad
\end{equation}
where:

\begin{itemize}
\item(Initialization) $I$ is a sequence of initial assignments over $x_1,\ldots, x_m$.
  That is, $I$ is an assignments sequence $x_1 := c_1; x_2 :=
c_2; \dots x_m := c_m$, with $c_i\in\R$ representing a number drawn
from a known distribution~\footnote{a known distribution is a
  distribution with known and computable moments} - in particular,
$c_i$ can be a real constant.
\item (Update) $U$ denotes a sequence of $m$ random updates, each update of the form:
\begin{equation}\label{eq:ProbModel:prob_assignments}
x_i := a_i x_i + P_{i}(x_1,\dots x_{i-1}) \;[p_i]\; b_i x_i + Q_{i}(x_1,\dots x_{i-1}),
  \end{equation}
or, in case of a deterministic assignment, 
\begin{equation}\label{eq:ProbModel:det}
  x_i := a_i x_i + P_{i}(x_1,\dots x_{i-1}),
  \end{equation}
where $a_i, b_i \in \R$ are constants and $P_{i},
Q_{i}\in\R[x_1,\ldots,x_{i-1}]$ are polynomials over program
variables $x_1,\ldots,x_{i-1}$.
\item (Dependencies) The coefficients $a_i$, $b_i$  and the coefficients of $P_i$ and
$Q_i$ in the variable
assignments~\eqref{eq:ProbModel:prob_assignments}-\eqref{eq:ProbModel:det}
of $x_i$ can be drawn from a random distribution  as long as
the moments of this distribution are known and either they are  (i) 
Gaussian or uniform distributions linearly depending on $x_i$ and
other random
variables $x_j$ with $j\neq i$; or (ii) other known distributions independent from
$x_1,\ldots, x_m$ . 
\end{itemize}
\end{definition}

\begin{algorithm}[t]
\caption{Moment-Based Invariants (MBIs) of \ProbModel{}\label{algo:MBIs} Loops}
\hspace*{\algorithmicindent} \textbf{Input:} \ProbModel{} loop
$\mathcal{P}$ with variables
$\{x_1,\dots,x_m\}$, and $k\geq 1$  \\
\hspace*{\algorithmicindent} \textbf{Output:} $MBI$s of $\mathcal{P}$
of degree $k$ \\
\hspace*{\algorithmicindent} \textbf{Assumptions:} $n\in\N$ is an
arbitrary loop iteration of $\mathcal{P}$
\begin{algorithmic}[1]
  \State Extract moment-based recurrence relations of
  $\mathcal{P}$, for $i=1,\ldots,m$:
  \[\begin{array}{lcl}
\E[x_i(n+1)]
		&= & p_i\cdot  \E\big[a_i x_i(n) + P_i(x_1(n),\dots,x_{i-1}(n))\big]\\
 		&& +(1-p_i)\cdot \E\big[b_i x_i(n) +
           Q_i(x_1(n),\dots,x_{i-1}(n))\big].
      
    \end{array}\]
  \State $MBRecs=\{\E[x_i(n+1)]~\mid~ i=1,\ldots,m\}$\label{algo:MBRecs:init}\Comment{initial set of
    moment-based recurrences}
 \State
 $S:=\{x_1^k,\ldots,x_m^k\}$\label{algo:Evars:init}\Comment{initial
   set of monomials of $\E$-variables}
   \While{$S\not=\emptyset$}
        \State $M :=  \prod_{i=1}^{m}x_i^{\alpha_i} \in S$, where
        $\alpha_i\in\N$
        \State  $S := S\setminus\{M\}$
        
        \State $M' = M[x_i^{\alpha_i}\leftarrow upd_i]$,~ for each $i=m,\ldots,1$
        \Comment{replace each $x_i^{\alpha_i}$ in $M$ with $upd_i$}
        \Statex
        \Statex  \qquad\qquad where $upd_i$ denotes:
        \Statex \qquad\qquad$p_i\cdot \big(a_i x_i + P_{i}(x_1,\dots
        x_{i-1})\big)^{\alpha_i} + (1-p_i)\cdot\big( b_i x_i +
        Q_{i}(x_1,\dots x_{i-1})\big)^{\alpha_i}$
        \Statex
        \State Rewrite $M'$ as $M'=\sum
        N_j$ for monomials $N_j$ over $x_1,\ldots,x_m$\qquad
        \State {\bf Simplify moment-based recurrence $\E[M(n+1)] = \E[\sum
        N_j]$ \label{algo:Evars} using~\eqref{eq:Evar:simplifRules}-\eqref{eq:Evar:newSimplifRules}} 
        \Statex \Comment{$M(n+1)$ denotes $\prod_{i=1}^{m}x_i(n+1)^{\alpha_i}$}
        \State $MBRecs=MBRecs\cup\{\E[M(n+1)]\}$
        \Statex\Comment{add $\E[M(n+1)]$ to the set of moment-based recurrences}
        \For{each monomial $N_j$ in $M$}
              \If{$\E[N_j]\nin MBRecs$}\Comment{there is no
                moment-based recurrence for $N_j$} 
              \State $S=S\cup\{N_j\}$ \Comment{add $N_j$~to~$S$}
              \EndIf
        \EndFor
        
        \EndWhile
        \State {\bf end while}\label{eq:algo:loop_end}
  \State $MBI=\{\E[x_i(n)^k]-f_{x_i,k}(n)=0 \ \mid \ i=1,\ldots m\}$\label{algo:MI}
  \Statex\Comment{$f_{x_i,k}(n)$ is the closed   form solution of $E[x_i^k]$}
  \State \underline{\bf return}  $MBI$s of
  $\mathcal{P}$ for the $k$th moments of $x_1,\ldots,x_m$
\end{algorithmic}
\end{algorithm}


Note that \ProbModel{} loop support parametrised distributions, for
example one may have the uniform distribution $\Unif(d_1,d_2)$ with
arbitrary $d_1,d_2\in\R$ symbolic constants. Similarly, 
the probabilities $p_i$ in the probabilistic
updates~\eqref{eq:ProbModel:prob_assignments}
can be symbolic constants.
The restriction on random variable dependencies from
Definition~\ref{def:ProbModel} extends~\cite{probsolvable} by
allowing parameters of Gaussian and uniform random variables $x_i$ in
\ProbModel{} loop to be specified using previously updated program
variables $x_j$ and to depend on $x_i$ linearly. In Theorem~\ref{thm:MBIs} we
prove that this extension maintains the existence and  computability of higher-order
statistical moments of \ProbModel{} loops, allowing us to derive all
\emph{moment-based invariants} of \ProbModel{} loops of degree $k\geq 1$.

\begin{definition}[Moment-Based Invariants (MBIs)]
  Let $\mathcal{P}$ be a  \ProbModel{} loop and $n\in\N$ denote
  an arbitrary loop iteration of $\mathcal{P}$. Consider $k\in\N$ with
  $k\neq 0$. A \emph{moment-based
    invariant (MBI) of degree $k$ over $x_i$} of $\mathcal{P}$ is $\E[x_i(n)^k]=f_{x_i,k}(n)$, where $f_{x_i,k}:\N\to\R$ of $n$
  is a closed form expression denoting the $k$th (raw) higher-order
  moments of $x_i$, such that $f_{x_i,k}(b)$ depends only $n$ and the initial
  variable 
  values of $\mathcal{P}$. 
\end{definition}

In what follows, we consider an arbitrary \ProbModel{} loop
$\mathcal{P}$ and formalize our results relative to
$\mathcal{P}$. Further, we reserve $n\in\N$ to denote an arbitrary
loop iteration of $\mathcal{P}$. 
Note that MBIs of  $\mathcal{P}$ yield
functional representations of the $k$th higher-order moments of loop
variables $x_i$ at $n$. Hence, the MBIs
$\E[x_i(n)^k]=f_{x_i,k}(n)$  are
valid and invariant. 
In Algorithm~\ref{algo:MBIs} we show that MBIs of \ProbModel{} loops can always
be computed. 
As in~\cite{probsolvable},
the main ingredient of Algorithm~\ref{algo:MBIs}  are so-called
\emph{$\E$-variables} for capturing expected values and other
higher-order moments of loop variables of $\mathcal{P}$. 

\begin{definition}[$\E$-variables of \ProbModel{}
  Loops~\cite{probsolvable}\label{def:Evars}]
  An \emph{$\E$-variable} of
  $\mathcal{P}$ is an expected value of a monomial over the random
  variables $x_i$ of $\mathcal{P}$.
\end{definition}

Using Definition~\ref{def:Evars}, in Algorithm~\ref{algo:MBIs} we
compute $\E$-variables based on expected values $\E[x_i(n)]$ of loop
variables $x_i$, as well as using higher-order and mixed moments of
$\mathcal{P}$,  such as $\E[x_i^k(n)]$ or $\E[x_i x_j(n)]$
(lines~\ref{algo:Evars:init}  and~\ref{algo:Evars} of
Algorithm~\ref{algo:MBIs}). To this end, Algorithm~\ref{algo:MBIs}
resembles the approach of~\cite{probsolvable} and extends it to handle
\ProbModel{} loops with dependencies among random variables drawn from
Gaussian/uniform distributions (line~\ref{algo:Evars} of
Algorithm~\ref{algo:MBIs}). More specifically,
Algorithm~\ref{algo:MBIs} 
uses \emph{moment-based recurrences over
  E-variables} from~\cite{probsolvable},  describing
the expected values $\E[x_i(n)]$ of $x_i$ as
functions of other E-variables (line~\ref{algo:MBRecs:init} of Algorithm~\ref{algo:MBIs}). To this end, note that 
\ProbModel{} loop updates 
from~\eqref{eq:ProbModel:prob_assignments}-\eqref{eq:ProbModel:det}
over $x_i$
yield 
linear recurrences with constant
coefficients over $\E[x_i(n)]$, by using the following simplification
rules over $\E$-variables:
\begin{equation}\label{eq:Evar:simplifRules}
\begin{array}{lcl}
\E[expr_1 + expr_2] &\rightarrow& \E[expr_1] + \E[expr_2]\\
\E[expr_1 \cdot expr_2] &\rightarrow & \E[expr_1] \cdot \E[expr_2],
                                      \quad\ {\small \text{if\ } expr_1,expr_2\text{ are independent}}\\
\E[c\cdot expr_1] &\rightarrow & c\cdot \E[expr_1]\\
\E[c] &\rightarrow & c\\
\E[\mathcal{D} \cdot expr_1] &\rightarrow& \E[\mathcal{D}]\cdot \E[expr_1]
\end{array}
\end{equation}
where $c\in\R$ is a constant, $\mathcal{D}$ is a known
independent distribution,  and $expr_1$, $expr_2$ are polynomial expressions over random variables. 
%
%
Yet, to address our \ProbModel{} loop extensions compared
to~\cite{probsolvable}, in addition to~\eqref{eq:Evar:simplifRules} we
need to ensure that dependencies among the 
random variables of $\mathcal{P}$  yield also 
moment-based recurrences. We achieve this  by introducing the
following two simplification rules over random variables with
Gaussian/uniform distributions: 
\begin{equation}\label{eq:Evar:newSimplifRules}
\begin{array}{lcl}
\Gaussian(expr_1, \sigma^2)	&\rightarrow & expr_1+\Gaussian(0,\sigma^2),\\
\Unif(expr_1,expr_2)			&\rightarrow & expr_1 + (expr_2-expr_1)\Unif(0,1),\\
\end{array}
\end{equation} 
for arbitrary polynomial expressions $expr_1$, $expr_2$ over random
variables. 
Using~\eqref{eq:Evar:newSimplifRules} in addition to~\eqref{eq:Evar:simplifRules}, moment-based recurrences of
\ProbModel{} loops can always be computed as linear recurrences with
constant coefficients over $\E$-variables (line~\ref{algo:Evars} of Algorithm~\ref{algo:MBIs}), implying thus the existence
of closed
form solutions of $\E$-variables and hence of MBIs of $\mathcal{P}$,
as formalized below. 

\begin{theorem}[Moment-Based Invariants (MBIs) of \ProbModel{}\label{thm:MBIs} Loops]
Let $\mathcal{P}$ be  a \ProbModel{} loop with variables
$\{x_1,\ldots, x_m\}$ and consider $k\in\N$ with $k\geq 1$.
Algorithm~\ref{algo:MBIs} is sound and terminating, yielding MBIs
of degree $k$ of $\mathcal{P}$. 
\end{theorem}
\begin{proof}
  We first prove correctness of the simplification
  rules~\eqref{eq:Evar:newSimplifRules}, from which the soundness and 
  termination of Algorithm~\ref{algo:MBIs} follows. 
%
%
Recall that there is a one-to-one correspondence between
probability distributions and characteristic functions $E[e^{itX}]$ of a random
variable $X$.
In particular, the characteristic function of a Gaussian distribution
with parameters $\mu$ and $\sigma^2$ is $e^{i\mu t -
  \frac{1}{2}\sigma^2 t^2}$, and thus the characteristic function of
$\Gaussian(expr_1, \sigma^2)$ is $\E[e^{it\Normal(expr_1,
  \sigma^2)}]$. Then, 


\begin{align*} 
\E\left[e^{it\Normal(expr_1, \sigma^2)}\right] 
&= \E\left[\int    e^{it\Normal(y,\sigma^2)}  f(y)   dy\right]
                                                  =\iint e^{itx} \frac{1}{\sqrt{2\pi\sigma^2}} e^{-\frac{(x-y)^2}{2\sigma^2}} f(y)dxdy			 \\ 
                                                                                                                                                                       \\
& = \int e^{itx} \frac{1}{\sqrt{2\pi\sigma^2}} e^{-\frac{(x)^2}{2\sigma^2}} dx \int e^{ity} f(y) dy	\\		 
& = \E\left[e^{it\Normal(0,\sigma^2)}\right] \cdot \E\left[e^{it\cdot               expr_1}\right] = \E\left[e^{it\left(\Normal(0, \sigma^2) + expr_1\right)}\right]
\end{align*}
by change of limits for $x\in\R$, where $f$ is the probability density
function of the random variable $expr_1$. Note that 
$\E\left[e^{it\left(\Normal(0, \sigma^2) + expr_1\right)}\right]$
corresponds to the characteristic function of
$expr_1+\Gaussian(0,\sigma^2)$, and hence the simplification rule
$\Gaussian(expr_1, \sigma^2) \rightarrow  expr_1+\Gaussian(0,\sigma^2)$
of~\eqref{eq:Evar:newSimplifRules} is correct. The correctness of the 
simplification rule of~\eqref{eq:Evar:newSimplifRules} over  uniform
distributions can be established in a similar way. 


Further, observe that polynomial expressions remain polynomial after applications
of~\eqref{eq:Evar:newSimplifRules} (line~\ref{algo:Evars} 
of Algorithm~\ref{algo:MBIs}).
Once Gaussian and uniform distributions depending on loop variables
are replaced using~\eqref{eq:Evar:newSimplifRules},
we are left with independent known distributions and polynomial
expressions over random variables for
which~\eqref{eq:Evar:simplifRules} can further be used, as
in~\cite{probsolvable}. As Algorithm~\ref{algo:MBIs}
extends~\cite{probsolvable} only with~\eqref{eq:Evar:newSimplifRules}
(line~\ref{algo:Evars} of Algorithm~\ref{algo:MBIs}),
using  results of~\cite{probsolvable}, we
conclude that Algorithm~\ref{algo:MBIs} is both sound and terminating.
\qed
\end{proof}

\begin{example}
Consider the \ProbModel{} loop in Fig.~\ref{fig:rats}(B). An example of $\E$-variable would be $E[W2^2]$, 
for which an MBI 
$E[W2^2] = 4.01408a^2 + 53.83168a + 4.01408b + 250.3172$ 
is computed using Algorithm~\ref{algo:MBIs}.
\end{example}

\begin{remark}
  While \ProbModel{} loops are
  non-deterministic, with
  trivial loop guards of $true$,  we note that probabilistic loops
  bounded by a number of iterations 
  (such as $n:=0; while(n<1000)\{n:=n+1\}$) 
  can be encoded as \ProbModel{} loops. 
\end{remark}

%% file: encoding.tex

\section{Encoding BNs as \ProbModel{} Loops} 
\label{sec:encoding}


In this section we argue that \ProbModel{} loops offer a natural way
for encoding BNs, enabling further BN analysis via \ProbModel{}
loop reasoning in Section~\ref{sec:BN problems}.

\subsection{Modeling Local Probabilistic Models of BNs as
  \ProbModel{} Loop Updates} 
\label{ssec:LPMs}
A BN is fully specified by its local dependencies. 
We consider common local probabilistic models and
encode these models as \ProbModel{} loop instances, as follows.

\subsubsection{Deterministic Dependency}
We first explore local probabilistic models specifying deterministic
dependency, that is when the values of  BN nodes $X$  are determined by the
values of the parent variables from $Par(X)$.
For example, when $X$ is binary-valued, such a
deterministic dependency can be a Boolean expression. On the other
hand, when $X$ 
is continuous, 
deterministic dependency can be a function over $Par(X)$. 

For a continuous variable $X$ whose value is given by a polynomial $Q(Par(X))$, encoding deterministic dependencies
as a \ProbModel{} loop update is straightforward: we simply set $X = Q(Par(X))$.

For a discrete random variable $X$, let $[X=x]$ be the expression such
that $[X=x] = 1$ if $X=x$ and $0$ otherwise. Note that when $X$ is
binary-valued, we have $[X = 1] = X$ and $[X = 0] = 1-X$.
It follows that, in general, for a discrete variable $X$ with possible values $x=0, 1, \cdots, k$, we have 
$[X = x] = \prod_{{\substack{0\le i<k \\ i\neq d}}} \frac{X-i}{x-i}.$
Furthermore, let $[(X, Y) = (x, y)] = [X = x] \cdot [Y = y]$. Then, $[(X, Y) = (x, y)] = 1$ iff $X = x \land Y = y$, and $0$
otherwise. 
Finally, we write $[X\not=x]$ to denote $1 - [X=x]$. Observer that  $[X=x]$
and $[X\not=x]$ are polynomials in $X$, providing thus a 
natural way to specify deterministic dependencies as
updates~\eqref{eq:ProbModel:prob_assignments}-\eqref{eq:ProbModel:det}
of \ProbModel{} loops (see Algorithm~\ref{alg:BNtoPS}).

\subsubsection{Conditional Probability Tables -- CPTs}
\label{subsubsec:CPTs}
As shown in Fig.~\ref{fig:grass}(A), 
a common way to specify BN dependencies among discrete
variables is CPTs, with each CPT line
representing a possible assignment of values of a BN node $X$ to
$Par(X)$.  
A CPT for $X$ can be turned into \ProbModel{} loop updates,  as
follows.

We represent values of $X$ with integers. For simplicity, assume that
$X$ is binary-valued. Let $Par(X)=\{Y_1, \cdots, Y_k\}$, denoting the
parents of $X$.
For each line $L$ in the CPT for $X$ we introduce a new variable
$X_L$. Each line $L$ specifies values for $Par(X)$; for example,  $Y_1
= y_1, \cdots, Y_k = y_k$. Let $P(X | L) = p_L$ and define

\begin{equation}\label{eq:encode cpt}
	X_L = \prod_{0<i\le k}  [Y_i = y_i] [p_L] 0,
\end{equation}
encoding that the value of $X_L$ is $0$ if the values of $Y_i$ are
not specified in the respective CPT line $L$; otherwise the value
of $X_L$ is 1 with probability $p_L$. We then set 
\begin{equation}\label{eq:encode d}
	X = \sum_{L\in CPT} X_L.
      \end{equation}


 \begin{example}
   Using~\eqref{eq:encode cpt}-\eqref{eq:encode d},
     the disBN of Fig.~\ref{fig:grass}(A) is encoded as a
     \ProbModel{} loop in Fig.~\ref{fig:grass}(B). 
     While the parameters of $S$ and $G$ are not
 directly visible from the disBN, these parameters are given by the expected
 values of $S$ and $G$ in the \ProbModel{} loop of Fig.~\ref{fig:grass}(B). 
 Note that Fig.~\ref{fig:grass}(B) also features a $GR$ variable
 corresponding to a Bernoulli random variable depending on $G$ and
 $R$, such that $GR$ is $1$ iff both $G$ and $R$ are $1$. The 
program variable $continue$ samples a sequence of Bernoulli random
variables (one for each iteration $n$),
while the random variable $count$ represents a geometric distribution
encoding the sum of $continue$ values. 
\end{example}

\subsubsection{Linear Dependency for Gaussian Variables}

A local probabilistic model  for a Gaussian random variable with
continuous parents
(as introduced in Definition~\ref{def:BNvariants}) can be encoded as a
\ProbModel{} loops update, as follows:
\begin{equation}\label{eq:encode g}
	X = RV(gauss,\: \alpha_X + \sum_{Y\in Par(X)} \beta_{X,Y} \cdot Y,\: \sigma_X^2),
\end{equation}
where $\alpha_X, \beta_{X,Y}$ are constants, $\sigma_X^2$ is fixed and $RV(gauss,
\mu,\sigma^2)$ denotes  a Gaussian random variable drawn from a Gaussian
distribution $\Gaussian(\mu,\sigma^2)$. 


\subsubsection{Conditional Linear Gaussian Dependency} 

By combining BN dependencies on discrete and continuous variables for
a Gaussian random variable $X$,
we can model conditional linear Gaussian dependencies for $X$.
Let $D$ be the joint distribution of the discrete parents of $X$ and for each $d\in  D$
let $\Gaussian_d$ be the Gaussian distribution associated with condition $d$ (here $\Gaussian_d$ may depend on the
values of continuous parents $Par(X)$ of $X$, as discussed in Section~\ref{sec:probsolvable}). 
The conditional linear Gaussian dependency for $X$ can be modeled as
the following \ProbModel{} loop update: 
\begin{equation}\label{eq:encode clg}
	\sum_{d \in \Omega(D)} [D = d] \cdot N_d.
\end{equation}

\begin{example}
  Fig.~\ref{fig:rats}(B) shows the \ProbModel{} loop encoding of the 
  clgBN of Fig.~\ref{fig:rats}(A). The random variables, $W_1$ and
  $W_2$ are given by conditional linear Gaussian dependency and
  encoded using~\eqref{eq:encode clg}.
  For simplicity, $W1$ and $W2$ are further split into variables $W1\_1$ and $W1\_2$, and $W2\_1$ and $W2\_2$, respectively, representing different values of $W1$ and $W2$ based on the value of $D$.
  Further, $D1\_S$ is a binary variable which is $1$ iff $D$ is $1$
  and $S$ is $0$, and $W2D1\_S$ represents the expected value of $W2\cdot D1\_S$. 
 \end{example}

 \subsubsection{Temporal Dependencies in DynBNs}
Dependencies in dynBNs are given by intra- and inter-time-slice edges. 
While the encoding of these dependencies is similar to the
aforediscussed BN dependencies, 
there are two restrictions on the structure of the dynBNs 
ensuring that dynBNs can be
encoded as \ProbModel{} loops. 
First, dependency of a dynBN variable $X$ on itself must be represented by a
linear function. This restriction could be lifted for discrete
variables, as discussed in Lemma~\ref{lemma:dRV first moments}. 
Second, a variable $X$ can only depend on itself in previous time-slice and current time-slice variables. 

\begin{example}
Fig.~\ref{fig:dynBN}(B) lists the \ProbModel{} loop corresponding to Fig.~\ref{fig:dynBN}(A). 
The Bernoulli random variables $R$ and $U$ are encoded
using~\eqref{eq:encode cpt}-\eqref{eq:encode d}. The parameters of $R$
and $U$ change across iterations, corresponding to parameters in
different time-slices of the dynBN; their concrete values are given by
the expected values of $R$ and $U$. 
\end{example}




\begin{algorithm}[t]
\caption{Encoding  BN variants as \ProbModel{} loops\label{alg:BNtoPS}}

\hspace*{\algorithmicindent} \textbf{Input:} BN \\
\hspace*{\algorithmicindent} \textbf{Output:}  \ProbModel{} program\\
\hspace*{\algorithmicindent} \textbf{Notation:} $LPM$ denoting a local
probabilistic model
\begin{algorithmic}[1]
	\State $Nodes:=$ topologically ordered set of BN nodes 
	\For{$X$ in $Nodes$}
		\If{LPM of $X$ is CPT}
			\For{each line $L$ in the CPT}   Set
                        $X_L$ as in~\eqref{eq:encode cpt}
	    	\EndFor
	    	\State Set $X$ as in~\eqref{eq:encode d}
		\EndIf
		\If{LPM of $X$ is a linear dependency for Gaussian
                  variables} Set $X$ as in~\eqref{eq:encode g}
		\EndIf
		\If{LPM of $X$ is a conditional linear Gaussian dependency} Set $X$ as in~\eqref{eq:encode clg}
		\EndIf
	\EndFor
\end{algorithmic}
\end{algorithm}

\subsection{Encoding BNs as \ProbModel{} Loops}
Section~\ref{ssec:LPMs} encoded common 
local probabilistic models of BN dependencies as
\ProbModel{} loop updates. Since BNs are DAGs, BN nodes can be ordered in
such a way that each BN node $X$ depend only on previous BN variables---
its parents $Par(X)$. Hence, BNs can be encoded as \ProbModel{} loops,
as shown in Algorithm~\ref{alg:BNtoPS} and stated below.

\begin{theorem}\label{thm:BNtoPS}
Every BN and dynBN\footnote{subject to the restriction on structure of dynBN as discussed in sectin~\ref{ssec:LPMs}}
with local probabilistic models given by CPT or
(conditional linear) Gaussian dependencies can be encoded as a
\ProbModel{}~loop.
In particular, disBNs, gBNs and clgBNs can be encoded as  \ProbModel{}~loops. 
\end{theorem}

Based on Algorithm~\ref{alg:BNtoPS} and Theorem~\ref{thm:BNtoPS}, we
complete this section by defining the following class of BNs, in
relation to \ProbModel{} loops.

\begin{definition}[\ProbModel{} Bayesian Networks]
A \emph{\ProbModel{} Bayesian Network (PSBN)} is a BN which can be encoded as a \ProbModel{} loop.
\end{definition}

The relation and expressivity of PSBNs, and hence \ProbModel{} loops,
compared to BN variants is visualized in Fig.~\ref{fig:hierarchy}.

\begin{figure}[t!]
  \includegraphics[width=\linewidth]{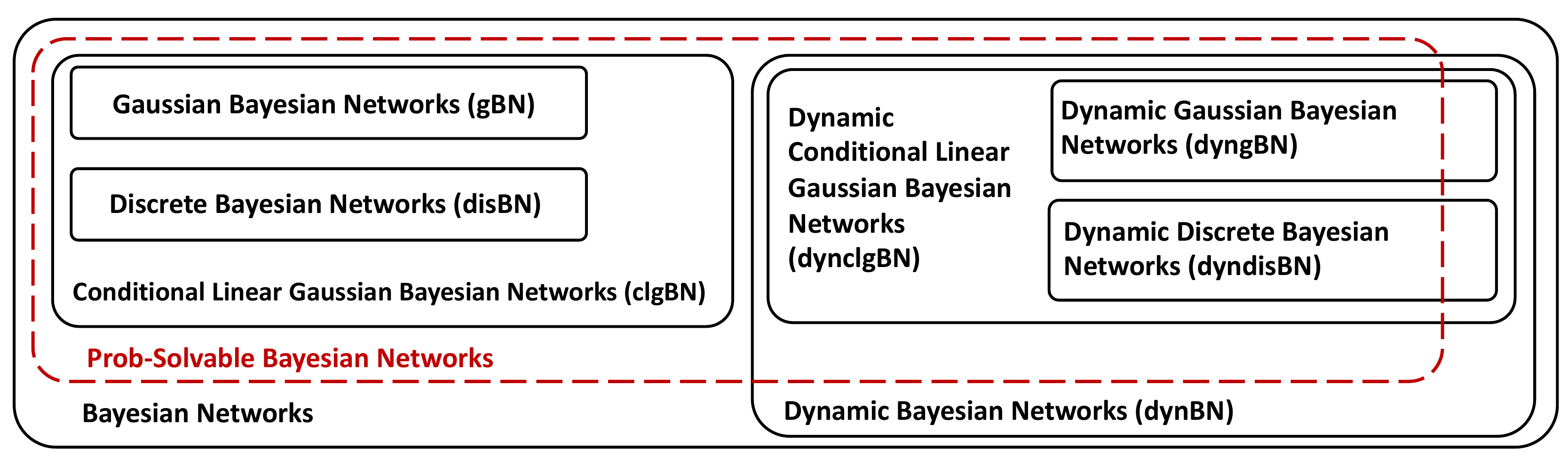}
   \vspace{-5ex}
  \caption{BN hierarchy.}
  \vspace{-5ex}
  \label{fig:hierarchy}
\end{figure}

%% file: solving_bn_problems.tex
\section{Automatic BN Analysis via \ProbModel{} Loop Reasoning}
\label{sec:BN problems}

We now show that several BN challenges can automatically solved by generating moment-based invariants of \ProbModel{} loops encoding the respective BNs.
To this end, (i) we consider exact inference, sensitivity analysis, filtering and computing the expected number of rejecting samples in sampling-based BN procedures
and (ii) formalize these BN problems as reasoning tasks within \ProbModel{} loop analysis.
We then (iii) encode BNs as \ProbModel{} loop $\mathcal{P}$ using Algorithm~\ref{alg:BNtoPS} and  (iv) generate moment-based invariants of $\mathcal{P}$ using Algorithm~\ref{algo:MBIs}.
We address steps (i)-(ii) in
Sections~\ref{subsec:infer}-\ref{subsec:sensitivity},   
and report on the automation of our work in Section~\ref{sec:implement}.

\subsection{Exact Inference in BNs}
\label{subsec:infer}

Common queries on BN properties address (i) the probability distributions of BN nodes $X$, for example by answering whether $P(X=x)$ or $P(X<c)$; 
(ii) the conditional probabilities of BN nodes $X, Y$, such as $P(X=x | Y = y)$;  or (iii) the expected values and higher-order moments of BN nodes $X,Y$, for instance  $\E[X], \E[X^2], \E[X | Y = y]$ and $\E[X^2 | Y = y]$. Here we focus on (iii) but show that, in some BN variants, queries related to (ii) can also be solved by our work.

\subsubsection*{Exact Inference in disBNs} 
In the case when a BN node $X$ is binary-valued, we have $\E[X] = P(X = true)$. Furthermore, for any higher-order moment of $X$ we also have $Mom_k[X] = P(X = true)$. 
For non-binary-valued but discrete BN node $X$, with values from $\{0, \dots, m\}$, the higher-order moments of $X$ are also computable. Moreover, the first $m-1$ moments are sufficient to fully specify probabilities $P(X=i)$,  for $i\in\{0, \dots, m-1\}$, as proven below. 
\begin{lemma}
\label{lemma:dRV first moments}
The higher-order moments of a discrete random variable $X$ over $\{0, \dots, m-1\}$ are specified by the  first $m-1$ higher-order moments of $X$. 
\end{lemma}
\begin{proof}
  Let $P(X=i) = p_i$,  for $i\in\{0, \dots, m\}$. Then, $\sum_{0\le i <m} i^kp_k = Mom_k(X)$, yielding $m$ linear equations over $p_0, \cdots p_{m-1}$, with
$k\in\{1, \cdots m-1\}$. As $\sum_{0\le i <m} p_i = 1$, we have  a linear system of $m$ linearly independent equations, implying the existence of a unique solution which specifies the distribution of $X$. \qed
\end{proof}

For computing conditional expected values and higher-order moments, we show next that deriving $\E[X^k | D=i]$ is reduced  to the problem of computing $\frac{\E[X^k\cdot[D=i]]}{\E[[D=i]]}$. 
\begin{lemma}
\label{conditional moments rewriting}
If $D=i$ with non-zero probability, we have  $\E[X^k | D=i]=\frac{\E[X^k\cdot[D=i]]}{\E[[D=i]]}$.
\end{lemma}
\begin{proof}
By partition properties for expected values, we have 
$${\E[X^k[D=i]]} ={\E[X^k[D=i] | D=i] P(D=i)} +{\E[X^k[D\not=i] | D=i] P(D\not=i)}.$$ 
As $[D=i] = 1$ iff $D=i$, we derive ${\E[X^k[D=i] | D=i]} = {\E[X^k | D=i]}$ and $\E[X^k[D\not=i] | D=i] = 0$. Therefore, $\E[X^k | D=i] = {\E[X^k | D=i] P(D=i)}$. Since $P(D=i) \not= 0$, we conclude $\E[X^k | D=i] = \frac{\E[X^k\cdot[D=i]]}{\E[[D=i]]}$. \qed
\end{proof}

\subsubsection*{Exact Inference in gBNs}
Recall that a Gaussian distribution is specified by its first two moments, that is by its mean $\mu$ and variance $\sigma^2$. 
As all nodes in a gBN are Gaussian random variables, the first two moments of gBN nodes are sufficient to analyse gBN behaviour. Further, 
$\E[X]$ and $\E[X^2]$ of a gBN node $X$ are computable using Algorithm~\ref{algo:MBIs}. 


\subsubsection*{Exact Inference in clgBNs}
As continuous variables $X$ in clgBNs are Gaussian random variables, the means and variance of $X$ are also computable using Algorithm~\ref{algo:MBIs}. 
However, clgBNs might also include discrete variables $D$, whose (conditional) higher-order moments can be computed as in Lemmas~\ref{lemma:dRV first moments}-\ref{conditional moments rewriting}.
Further, for a continuous variable $X$ and a discrete variable $D$ in a clgBN, we have 
\[ \E[X | D=i] = \frac{\E[X^k\cdot[D=i]]}{\E[[D=i]]},  \]
allowing us, for example, to derive $\E[W2 | D = 1] = 7.25 + 0.89a$ in Fig.~\ref{fig:rats}.

\subsubsection*{Exact Inference in dynBNs}
%
As dynBNs are infinite in nature, (infinite) \ProbModel{} loops are suited to reason about dynBN inferences, such as (i) 
long-term behaviour or prediction and (ii) filtering and smoothing. 
A related problem is  characterizing the dynBN behaviour after $n$ iterations, and in particular for $n\to\infty$.

\paragraph{(i) Prediction and long-term behaviour in dynBNs}
By modeling dynBNs as \ProbModel{} loops, we can
compute/predict higher-order moments $\E[X_n^k]$  of dynBN nodes $X$ using Algorithm~\ref{algo:MBIs}, for an arbitrary $n$.
Further, thanks to the existence  of $\E[X_n^k]$ for \ProbModel{} loops, we conclude that $\lim_{n\to\infty} \E[X_n^k]$ is also computable. 
Moreover,
Algorithm~\ref{algo:MBIs} computes higher-order moments/MBIs in $O(1)$ time w.r.t. $n$, which is not the case of the $O(n)$ approach of the standard Forward algorithm.

\paragraph{(ii) Filtering and prediction in dynBNs}
Predicting next dynBNs states $X_{t+1}$ given all observations $e_1, \ldots, e_{t+1}$ until time $t$
can be expressed as $P(X_{t+1} | e_1, \dots, e_{t+1})$, which in turn can be rewritten using Bayes' rule under the sensor Markov assumption (the evidence $e_t$ depends only on program variables $X_t$ from the same time-slice), as follows:
\[ P(X_{t+1} | e_1, \dots, e_{t+1}) = P(e_{t+1} | X_{t+1}) \cdot \sum_{x_t} P(X_{t+1} | x_t) \cdot P(x_{t} | e_1, \dots, e_{t}), \]
where $P(e_{t+1} | X_{t+1})$ and $P(X_{t+1} | x_t)$ are specified by the BN,  assuming discrete-valued observation variables.
Filtering and prediction in dynBNs is thus computable using MBIs of \ProbModel{} loops.

\subsection{Number of BN Samples until  Positive BN Instance} 
\label{subsec:number of samples}



As pointed out in \cite{howlong}, an interesting question about BNs is "Given a Bayesian network with observed evidence, how long does it take in expectation to obtain a single sample that satisfies the observations?". A related, though arguably simpler, question would require giving the expected number of positive instances (samples satisfying the observation) in $N$ samples of BNs. 
Both of these questions can be answered using standard results from probability theory. 
\begin{lemma}
\label{thm:positive instances}
Given the probability $p$ of a BN observation, the expected number of positive $BN$ instances in $N$ samples is $p N$.
Further, the expected number of BN samples until  the first positive BN instance is $\frac{1}{p}$.
\end{lemma}
\begin{proof}
  Since every BN iteration (sample) is independent from previous ones, the occurence of positive BN instances
  can be modelled as a Bernoulli random variable,  given by the probability $p$ of positive instances in any given iteration (or sample).
The number of positive instances in $N$ samples is therefore the sum of independent, identically distributed Bernoulli random variables, parametrized by $p$, following thus a  Binomial distribution with parameters $N$ and $p$. The number of positive BN samples is thus $\E[Binom(N,p)] = p N$.
The expected number of BN samples until the first positive BN instance is therefore  given by the distribution of the number of Bernoulli trials needed for one success, which in turn is given by the geometric distribution $Geometric(p)$. The expected number of samples until the first positive BN instance is thus  $\E[Geometric(p)] = \frac{1}{p}$. \qed
\end{proof}

We note that Lemma~\ref{thm:positive instances} can be answered using  \ProbModel{} loop reasoning, by relying on Algorithm~\ref{algo:MBIs}, as illustrated next.

\begin{example} 
  For inferring the expected number of positive instances in $N$ samples in 
 Fig.~\ref{fig:grass}, we first encode the observation in the BN as a new variable $GR = G\cdot R$, capturing the observation that the grass is wet and there was  rain. We then transform the BN into a dynBN adding an inter-time-slice counter update $count = count + GR$. The expected number of positive instances is then the prediction $\E[count_n]$ for $n=N$. 

 For answering the question of~\cite{howlong}, we again encode the observation first as above, e.g. $GR = G\cdot R$. 
We use a boolean variable to indicate whether there has been a positive instance 
$continue = {continue\cdot [GR=0]}$, which is initiated as $1$ (or $true$) and updated to $0$ once $GR = 1$ 
and stays $0$ thereafter. Finally, we update a loop counter as long as there was no positive instance observed with 
$count = count + continue$. The expected number of samples until the first positive instance is the long-term behaviour of 
$count$, i.e. $\lim_{n\to\infty} \E[count_n]$. 
\end{example}


\subsection{Sensitivity Analysis in BNs}
\label{subsec:sensitivity}

As BNs rely on network parameters, 
a challenging task is to understand to what extent does a small change in a network parameter affect the outcome of particular BN query.
This task is referred to as sensitivity analysis in BNs. 
More precisely, we would like to compute 
 $P(X|e)$ and $\E[X|e]$ for a random variable $X$ and evidence $e$ as functions of a BN parameter(s) $\theta$. 
 For doing so, we note that \ProbModel{} loops may use symbolic coefficients. Thus, replacing concrete BN probabilities with symbolic parameters and solving BN queries as discussed in Section~\ref{subsec:infer}, allow us to automate sensitivity analysis in BNs by computing MBIs of the respective \ProbModel{} loops, using Algorithm~\ref{algo:MBIs}.

%
 
 \begin{example}
 A sensitivity analysis in Fig.~\ref{fig:rats} could measure the effect of parameters of weight loss in week $1$ on the conditional expectation $\E[W2|D=1]$. That is, we compute $\E[W2|D=1]$ as a function of parameters of $W1$. 
 In this case, we introduce symbolic parameters $a$ and $b$ adjusting the parameters of weight loss in week 1 ($W1\_1$) when drug was administers. 
 Using Algorithm~\ref{algo:MBIs}, we compute the MBIs 
 $\E[W2^k\cdot D], \E[D]$, from which we have, for $k=1$, 
 $\E[W2|D=1] = \frac{\E[W2\cdot D]}{\E[D]} = 0.89a + 7.25$, 
 answering the respective sensitivity analysis of Fig.~\ref{fig:rats}.

%
  \end{example}

%% file: implementation.tex

\subsection{Implementation and Experiments}\label{sec:implement}
We automated BN analysis via \ProbModel{} loop reasoning by extending
and using our
tool \Mora{}~\cite{mora}. To this end, we first manually encoded BNs as
\ProbModel{} loops using Algorithm~\ref{alg:BNtoPS}. We then extended \Mora{} to support our extended
programming model of \ProbModel{} loops and integrated 
Algorithm~\ref{algo:MBIs} within \Mora{}\footnote{\url{https://github.com/miroslav21/mora}} to generate MBIs of
\ProbModel{} loops, solving thus the BN problems of
Sections~\ref{subsec:infer}-\ref{subsec:sensitivity}. 
As benchmarks, we used $28$ BN-related problems for $6$ BNs taken
from~\cite{rats,marks,korb,asia,russell2002artificial}.
Table~\ref{tab:experiments} summarizes our experiments, with full details on the experimental data for Fig.~\ref{fig:gBN} in the Appendix~\ref{sec:app}. For each example of Table~\ref{tab:experiments},
we  list the BN queries we considered, that is probabilistic inference
(Q1),
number of BN samples (Q2) and sensitivity analysis (Q3) as introduced
in Section~\ref{sec:intro} and discussed in
Sections~\ref{subsec:infer}-\ref{subsec:sensitivity}. 
Column~3 of Table~\ref{tab:experiments}   shows the time needed by
\Mora{} to compute moment-based invariants (MBIs) solving the respective BN
problems. The last column of Table~\ref{tab:experiments} gives our 
derived solutions for the considered BN queries. 
Our experiments were run on a MacBook Pro 2017 with 2.3 GHz Intel Core i5 and 8GB RAM.

A symbolic version of gBN from Fig.~\ref{fig:gBN} was analysed in the experiments for sensitivity analysis (Table~\ref{tab:experiments}), with $\mu_{al}$ for the mean of $ALG$ variable, $\sigma_{an}$ for the variance of $ANL$, and $(0.31+c)$ for the coefficient of $ANL$ in the mean of $S$. Results were shortened in the Table~\ref{tab:experiments}. Full results are as follows:
\begin{align*}
\E[Stat^2] &= 0.9801\mu_{al}^2c^2 + 2.112462\mu_{al}^2c + 1.13827561\mu_{al}^2 - 7.0686\mu_{al}c^2 \\
  &\quad - 31.965132\mu_{al}c - 26.23869846\mu_{al} + c^2\sigma_{an} + 123.30018c^2 \\
  &\quad + 0.62c \sigma_{an} + 326.0841516c + 0.0961\sigma_{an} + 438.406319698 \\
\E[AverageMark] &= 0.33\mu_{al} c + 1.01896666666667\mu_{al}  - 1.19c - 5.2889 \\
\E[AverageMark^2] &= 0.1089\mu_{al}^2c^2 + 0.672518\mu_{al}^2c + 1.0382930677778\mu_{al}^2 \\
  &\quad - 0.7854\mu_{al}c^2 - 5.91581466666667\mu_{al}c \\
  &\quad - 10.7784256066667\mu_{al} + 0.111111111111111c^2\sigma_{an}  \\
  &\quad + 13.70002c^2 + 0.291111111111111c\sigma_{an}  + 88.4476124c \\
  &\quad  + 0.190677777777778\sigma_{an} + 162.736365699778.
\end{align*}

\begin{table}[h!]
\begin{tabular}{ll|l|l}
\hline
\rowcolor[HTML]{CBCEFB} 
{\scriptsize BN }               & {\scriptsize BN Problem }                            & {\scriptsize MBIs }                              & {\scriptsize BN Solutions }              \\ \hline
\rowcolor[HTML]{EFEFEF} 
\multicolumn{4}{l}{\cellcolor[HTML]{EFEFEF}{\color[HTML]{333333} \scriptsize{Grass -- Fig.~\ref{fig:grass} (disBN) \#nodes: 3, \#edges: 3, \#parameters: 7, \#variables in \ProbModel{} encoding: 9}}} \\ \hline
                  &    {\scriptsize Q1: $P(R|G)$  }        &     {\scriptsize $0.72s$  }    & {\scriptsize $P(R|G) = 0.8752$	}                          \\
                  &   {\scriptsize  Q2: Number of samples}          &   {\scriptsize  $1.24s$	} & {\scriptsize $\#samples = 1.37$	} \\  
                  &   {\scriptsize  Q3: Sensitivity analysis}          &   {\scriptsize  $0.82s$	} & {\scriptsize $\frac{0.04b + 0.6396}{-0.178a + 0.04b + 0.7308}$	}        \\ \hline
\rowcolor[HTML]{EFEFEF} 
\multicolumn{4}{l}{\cellcolor[HTML]{EFEFEF}{\color[HTML]{333333} \scriptsize{Alarm~\cite{russell2002artificial} -- Fig.~\ref{fig:disBN} (disBN) \#nodes: 5, \#edges: 4, \#parameters: 10, \#variables in \ProbModel{} encoding: 13}}} \\ \hline
                  &   {\scriptsize Q1: $P(B | A )$ }                  &{\scriptsize $0.83$s }                           & {\scriptsize $P(B | A ) = 0.373551$ }                  \\
                  &   {\scriptsize Q1: $P(EQ | M)$ }                &     {\scriptsize $1.01s$ }                        &   {\scriptsize $P(EQ | M) = 0.0358809$  }                      \\
                  &    {\scriptsize Q1: $P(\lnot EQ \land \lnot B | A \land J)$ }       & {\scriptsize $1.53s$ }     &   {\scriptsize $P(\lnot EQ \land \lnot B | A \land J) = 0.396195$  }              \\
                  &    {\scriptsize Q1: $P(EQ \land \lnot B | M \land J)$  }        &     {\scriptsize $1.43s$  }    & {\scriptsize $P(EQ \land \lnot B | M \land J) = 0.175492$	}                          \\
                  &   {\scriptsize  Q2: Number of samples (for $M\land J$)}          &   {\scriptsize  $1.91s$	} & {\scriptsize $\#samples = 19.978$	}       \\  
                  &   {\scriptsize  Q3: Sensitivity analysis (all of above)}          &   {\scriptsize  $3.36s$	} & {\scriptsize $P(B | A ) = \frac{b(0.01q + 0.94)}{-0.279bq + 0.939b + 0.289q + 0.001}, \cdots$ }                           \\ \hline
\rowcolor[HTML]{EFEFEF} 
\multicolumn{4}{l}{\cellcolor[HTML]{EFEFEF}{\color[HTML]{333333} \scriptsize{Asia~\cite{asia}  (disBN) \#nodes: 8, \#edges: 8, \#parameters: 18, \#variables in \ProbModel{} encoding: 24}}} \\ \hline
                  &   {\scriptsize Q1: $P(Asia, Lung | Dysp)$ }                  &{\scriptsize $2.25$s }           & {\scriptsize $P(Asia, Lung | Dysp ) = 0.00045596785$}                   \\
                  &   {\scriptsize Q2: Number of samples }                  &{\scriptsize $2.85$s }               & {\scriptsize $\#samples = 1818.1818$}                  \\
                  &   {\scriptsize  Q3: Sensitivity analysis}          &   {\scriptsize  $3.76s$	} & {\scriptsize $P(Asia, Lung | Dysp ) = \frac{0.192a + 0.29625b + 0.0221625}{0.992a + 0.62b + 48.6054} $}                          \\ \hline   
\rowcolor[HTML]{EFEFEF} 
\multicolumn{4}{l}{\cellcolor[HTML]{EFEFEF}\scriptsize{Marks~\cite{marks} -- Fig.~\ref{fig:gBN} (gBN) \#nodes: 3, \#edges: 3, \#parameters: 6, \#variables in \ProbModel{} encoding: 5-6}}                          \\ \hline
                  &  {\scriptsize Q1: Marks - expected values  }                 						&   {\scriptsize $0.05s$ }            &  {\scriptsize $\E[Stat] = 41.688, \cdots$}           \\
                  &  {\scriptsize Q3: Marks - sensitivity analysis EVs  }  & {\scriptsize $0.12s$ }   &  {\scriptsize $\E[Stat] = 0.99\mu_{al}c + 1.0669\mu_{al} - 3.57c - 12.2967, \cdots$}  \\
                  &  {\scriptsize Q1: Marks - second moments }   				&  {\scriptsize $0.11s$  }  &   {\scriptsize $\E[Stat^2] = 2035.718, \cdots$}         \\
                  &  {\scriptsize Q3: Marks - sensitivity 2nd moments}  &  {\scriptsize $0.28s$  }  &   {\scriptsize 
                  $\E[Stat^2] = 0.9801\mu_{al}^2c^2 + \cdots + 0.0961\sigma_{an} + 438.4063, \cdots$}         \\
                  &  {\scriptsize Q1: Average - expected values   }           	&   {\scriptsize $0.08s$ }           &  {\scriptsize $\E[AverageMark] = 46.271$}            \\
                  &  {\scriptsize Q3: Average - sensitivity EV } 	&   {\scriptsize $0.18s$ }             &  {\scriptsize $\E[AverageMark] = 0.33\mu_{al} c + 1.01897\mu_{al}  - 1.19c - \cdots$}            \\
                  &  {\scriptsize Q1: Average - second moments}  				&    {\scriptsize $0.13s$  }   	&   {\scriptsize $\E[AverageMark^2] = 2673.160$}            \\ 
                  &  {\scriptsize Q3: Average - sensitivity 2nd moment}  & {\scriptsize $0.46s$  }  & {\scriptsize 
                  $\E[AverageMark^2] = 0.1089\mu_{al}^2c^2 + \cdots + 0.190678\sigma_{an}$}            \\ \hline
\rowcolor[HTML]{EFEFEF} 
\multicolumn{4}{l}{\cellcolor[HTML]{EFEFEF}\scriptsize{Rats~\cite{rats} -- Fig.~\ref{fig:rats} (clgBN) \#nodes: 4, \#edges: 4, \#parameters: 11, \#variables in \ProbModel{} encoding: 10}}           \\ \hline
                  &   {\scriptsize Q1: $\E[W2 | D]$ }                           &     {\scriptsize $0.49s$  }         &  {\scriptsize $\E[W2 | D] = 15.02$ }                           \\
                  &   {\scriptsize  Q3: $\E[W2 | D]$	sensitivity }      &    {\scriptsize $0.72s$   }          &   {\scriptsize  $\E[W2 | D] = 15.02 + 2.24a$}                          \\
                  &    {\scriptsize Q1: $\E[W2^2 | D]$ }                        &    {\scriptsize $1.05s$  }         &    {\scriptsize $\E[W2^2 | D] = 242.8356$}  \\
                  &     {\scriptsize Q3: $\E[W2^2 | D]$ sensitivity  }     &    {\scriptsize $1.35s$   }        &     {\scriptsize $\E[W2^2 | D] = 242.8356 + 5.0176b + 67.2896a + 5.0176a^2$}                        \\ \hline
\rowcolor[HTML]{EFEFEF} 
\multicolumn{4}{l}{\cellcolor[HTML]{EFEFEF}\scriptsize{Umbrella~\cite{russell2002artificial} -- Fig.~\ref{fig:dynBN} (dynBN) \#nodes: 2, \#edges: 2, \#parameters: 3, \#variables in \ProbModel{} encoding: 6}}                     \\ \hline
                  &  {\scriptsize Q1: Prediction }                            &   {\scriptsize \multirow{2}{*}{$0.56s$} }	&  {\scriptsize	$\E[R] = \frac{1}{2}((2/5)^n+1) $ } \\
                  & {\scriptsize  Q1: Long-term behaviour }            &	&   	{\scriptsize $\E[R] \to \frac{1}{2}$	as $n \to \infty$}\\
                  &  {\scriptsize Q3: Prediction - sensitivity }          & {\scriptsize \multirow{2}{*}{$1.15s$}}	& {\scriptsize $\E[R] = \frac{(r - 1)(r - 0.3)^n-0.3}{r-1.3} $}\\
                  &  {\scriptsize Q3: Long-term - sensitivity }         &	&   	 	{\scriptsize $\E[R] \to \frac{0.3}{1.3-r}$	as $n \to \infty$}\\                      
\end{tabular}
\caption{BN analysis via \ProbModel{} loop reasoning within \Mora{}\label{tab:experiments}.}
\end{table}

\begin{figure}[h]
\centering
  \includegraphics[width=0.95\linewidth]{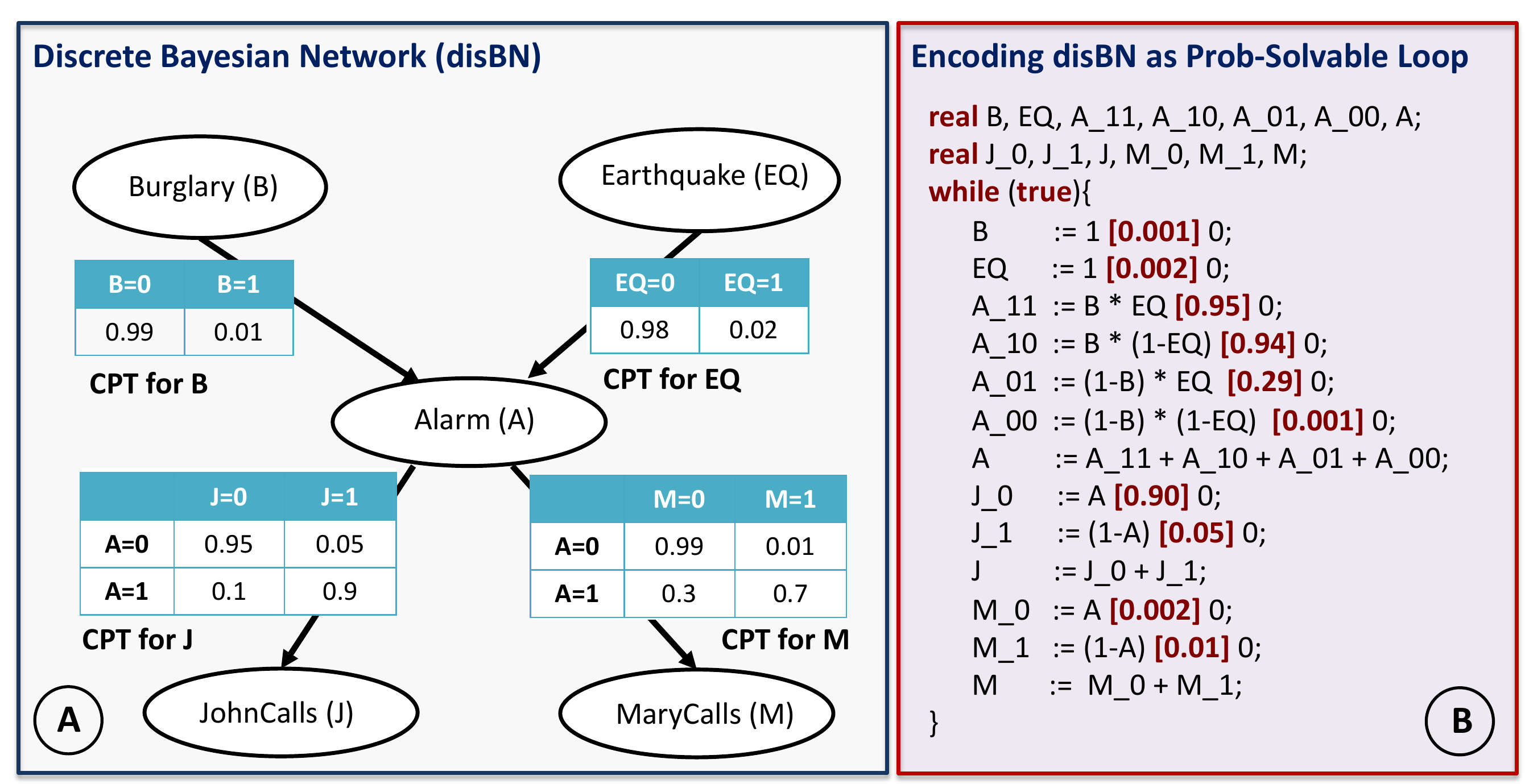}
  \vspace{-2ex}
  \caption{The discrete Bayesian Network (disBN) of
    Fig.~\ref{fig:disBN}(A) shows a burglar 
  alarm example. A burglar (B) and earthquake (EQ) 
  directly affect the probability of the Alarm (A) going 
  off, but whether or not John calls (J) or Mary calls (M)
  depends only on the alarm. A
  \ProbModel{} loop encoding for this disBN is given in Fig.~\ref{fig:disBN}(B).\label{fig:disBN}}
\end{figure}
  \vspace{-3ex}

\begin{figure}[!h]
\centering
  \includegraphics[width=0.95\linewidth]{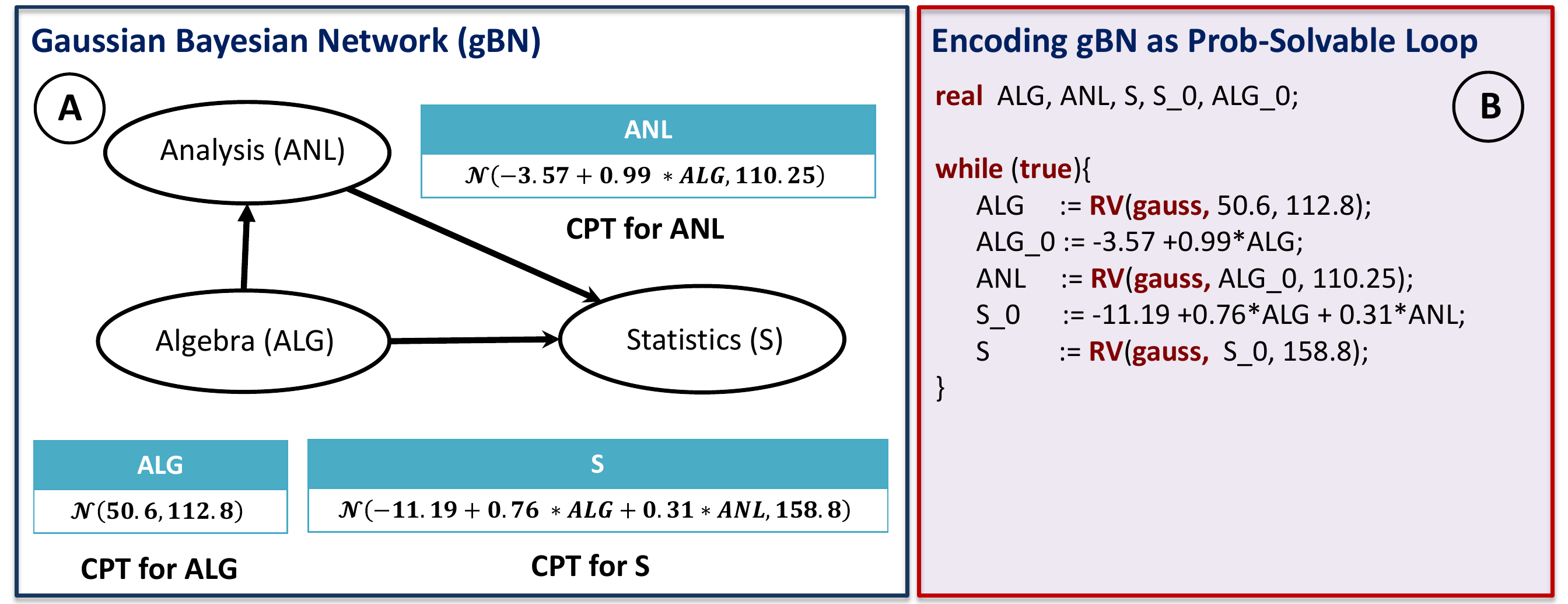}
  \vspace{-2ex}
  \caption{The Gaussian Bayesian Network (gBN) of
    Fig.~\ref{fig:gBN}(A)  
  describes the relationships between the marks on three math-related 
  topics. Its respective  \ProbModel{} loop encoding is given in Fig.~\ref{fig:gBN}(B).\label{fig:gBN}  }
\end{figure}

\begin{figure}[h]
  \includegraphics[width=\linewidth]{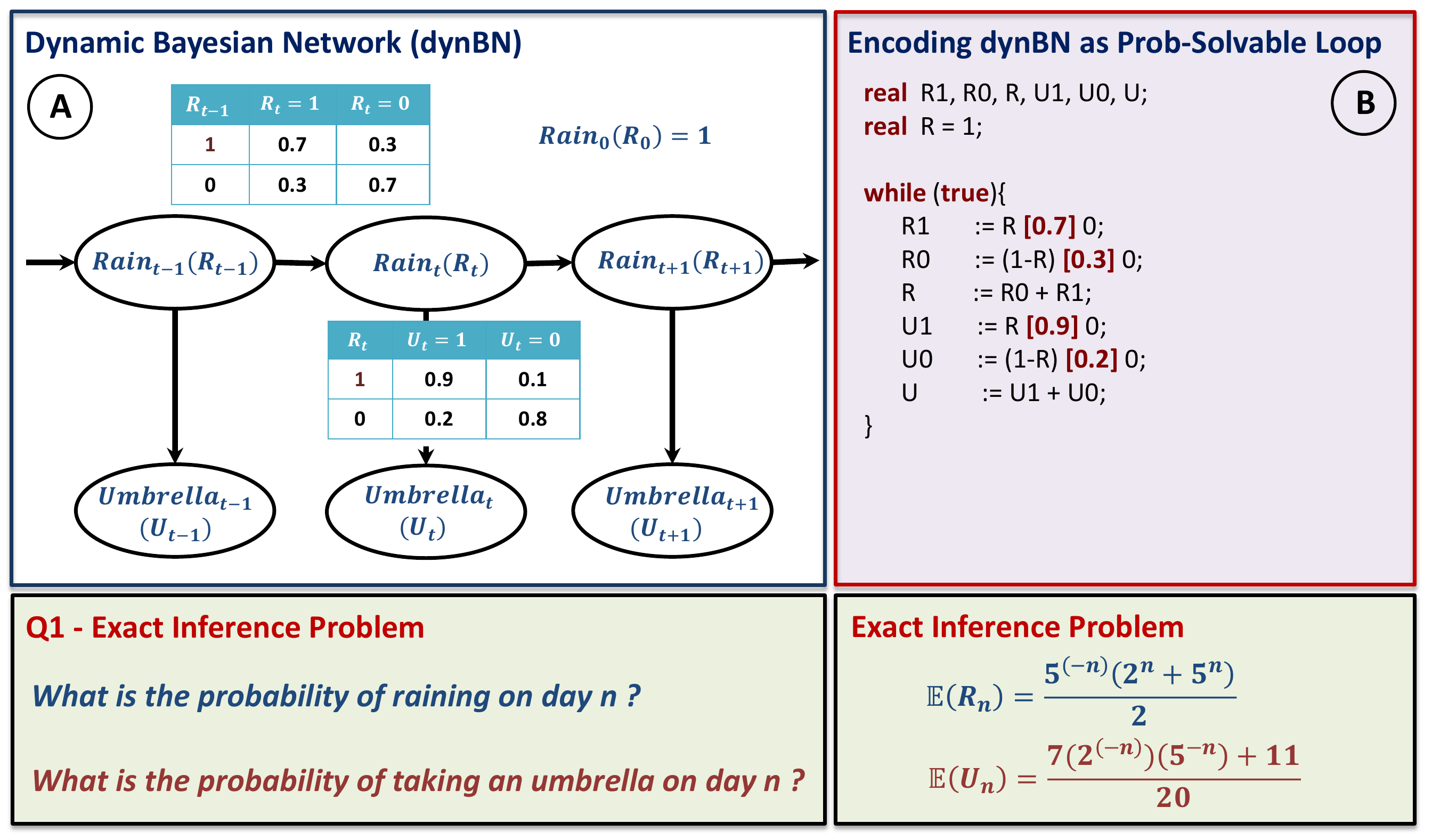}
  \vspace{-4ex}
  \caption{In Fig.~\ref{fig:dynBN}(B) we give the 
    \ProbModel{} loop encoding of the dynamic Bayesian Network (dynBN)
    from Fig.~\ref{fig:dynBN}(A). Solutions of probabilistic
    inferences in this dynBNs are also given, by computing MBIs of Fig.~\ref{fig:dynBN}(B).\label{fig:dynBN}}
\end{figure}

%% file: related.tex
\section{Related Work}

The classical approach to analyze probabilistic models is based on 
probabilistic model checking~\cite{Baier2008}. 
However, 
approaches~\cite{KwiatkowskaNP11,DehnertJK017,KatoenZHHJ11} cannot yet handle unbounded and real variables
that are required for example to encode Gaussian BNs, nor do they
support invariant generation, which is a key step in our work. 

In the context of probabilistic programs (PPs), a formal semantics for
PPs was first introduced in~\cite{Kozen81}, together with a deductive calculus to reason about 
expected running time of PPs~\cite{Kozen85}. This approach was further
refined and extended  in~\cite{McIverM05}, by introducing weakest
pre-expectations based on the
{weakest precondition} calculus of~\cite{Dijkstra75}. While~\cite{McIverM05}
infers quantitative invariants only over expected values of program variables, 
our moment-based invariants yield quantitative invariants over
arbitrary higher-order moments, including expected values. Further, the
setting of~\cite{McIverM05} considers PPs where the 
stochastic inputs are restricted to discrete distributions with 
finite support. To encode Gaussian BNs it is however necessary 
to handle also continuous  distributions with infinite support, as
described in our work. 

The first semi-automatic and complete 
method synthesizing the linear quantitative invariants needed
by~\cite{McIverM05} was introduced in~\cite{Katoen2010}. To this end,
PP loops are annotated with linear template invariants and constraint
solving is used 
to find concrete values of the template parameters. Further extensions
for template-based non-linear quantitative invariant generation have
been proposed in~\cite{Chen2015,Feng2017}. A related line of 
research is given in~\cite{Barthe2016}, where
martingales and user-provided hints are used to compute quantitative
invariants of PPs. 
The recent work of~\cite{Kura19} generalizes the use  of
martingales in conjunction with templates for computing higher-order moments of program
variables, with the overall goal of approximating runtimes of
randomized programs.
Unlike these works, our approach extends 
\ProbModel{} loops from\cite{probsolvable} and provides a fully
automated approach for deriving non-linear invariants 
over higher-order moments. 

Several techniques infer runtimes 
and expected values of PPs, see
e.g.~\cite{Monniaux01,Celiku2005,Hehner2011,FioritiH15,Bradzil2015}. 
To the best of our knowledge, however only~\cite{howlong} targets explicitly BNs on the
source code level, by using a weakest precondition calculus similar
to~\cite{KaminskiKMO16,McIverM05}. 
The PPs addressed in~\cite{howlong} are expressed in the
\emph{Bayesian Network Language} (BNL) fragment of the
\emph{probabilistic Guarded Command Language (pGCL)}
of~\cite{McIverM05}.
The main restriction of BNL is that loops prohibit undesired 
data flow across multiple loop iterations: it is not possible to assign to a variable the value of the 
same variable or another variable at the previous iteration.
Furthermore, BNL does not natively allow to draw samples
from Gaussian distribution, allowing thus only discrete BNs to be
encoded in BNL. In contrast  to~\cite{howlong}, in our work we use 
\ProbModel{} loops, as a subclass of PPs, to allow polynomial updates
over random variables and parametric distribution, Variable updates of \ProbModel{} loops can involve coefficients
from Bernoulli, Gaussian, uniform and other distributions, whereas variable updates drawn from Gaussian 
and uniform distributions can depend on other program
  variables. 
Compared to~\cite{howlong}, we thus support reasoning about  (conditional
linear) Gaussian BNs and our PPs also allow data flow 
across loop iterations which is necessary to encode 
dynamic BNs.

%% file: conclusion.tex
\section{Conclusion}
\label{sec:conclusion}
We  extend the class of \ProbModel{} loops with variable
updates over Gaussian and uniform random variables depending on other
program variables. We show that moment-based invariants (MBIs) in
\ProbModel{} loops can always be
computed as quantitative invariants over higher-order moments of loop
variables.
We further encode BN variants as \ProbModel{} loops, allowing us to turn several BN problems into
the problem of computing MBIs of \ProbModel{} loops. In particular, we
automate the BN analysis of exact inference, sensitivity analysis, filtering and computing the expected 
number of rejecting samples in sampling-based procedures via
\ProbModel{} loop reasoning. As future work, we plan to further extend the class of \ProbModel{} loops with more complex
flow and arithmetic and address termination analysis of such loops.